%% file: RejecSampling.tex
\documentclass{article} 
\usepackage{times}
\usepackage{makecell}

\usepackage{array}
\usepackage{graphicx} 
\graphicspath{{./experim/figures/icml16_fig/}}
\usepackage{caption}
\usepackage{subcaption}
\usepackage{natbib}

\usepackage{algorithm}
\usepackage{algorithmic}

\usepackage{amsmath}
\usepackage{amssymb}
\usepackage{amsthm}
\usepackage[disable]{todonotes}
\usepackage[accepted]{icml2016}
\input{misosymbols}

\usepackage[colorlinks,
           linkcolor=red,
           citecolor=blue,
           urlcolor=magenta,
           linktocpage,
           plainpages=false]{hyperref}

\usepackage{tikz}
\usepackage{tkz-graph}
\usetikzlibrary{shapes, arrows, intersections}
\usepackage{pgfplots}
\pgfplotsset{compat=1.9}
\pgfmathdeclarefunction{gauss}{2}{%
  \pgfmathparse{1/(#2*sqrt(2*pi))*exp(-((x-#1)^2)/(2*#2^2))}%
}
\pgfmathdeclarefunction{quadratic}{3}{%
  \pgfmathparse{#1*(x-#2)^2 + #3)}%
}

\definecolor{babyblue}{rgb}{0.54, 0.81, 0.94}
\definecolor{citrine}{rgb}{0.89, 0.82, 0.04}

\definecolor{akcolor}{rgb}{0.73, 0.89, 0.95}

\definecolor{myred}{rgb}{0.89, 0.12, 0.04}

\newcommand{\Real}{\mathbb{R}}

\icmltitlerunning{Pliable rejection sampling}

\begin{document}

\twocolumn[
\icmltitle{Pliable Rejection Sampling}
\vskip -0.05in
\icmlauthor{Akram Erraqabi\footnotemark}{akram.er-raqabi@umontreal.ca}
\icmladdress{MILA, 
             Universit\'e de Montr\'eal, Montr\'eal, QC
H3C 3J7,  Canada}
\icmlauthor{Michal Valko}{michal.valko@inria.fr}
\icmladdress{INRIA Lille - Nord Europe, SequeL team, 40 avenue Halley 59650, Villeneuve d'Ascq, France}
\icmlauthor{Alexandra Carpentier}{carpentier@math.uni-potsdam.de}
\icmladdress{Institut f\"ur Mathematik, Univert\"at Potsdam, Germany, Haus 9
Karl-Liebknecht-Strasse 24-25,
D-14476 Potsdam}
\icmlauthor{Odalric-Ambrym Maillard}{odalric.maillard@inria.fr}
\icmladdress{INRIA Saclay - \^{I}le-de-France, TAO team, 660 Claude Shannon, Universit\'e Paris Sud, 91405 Orsay, France}

\icmlkeywords{boring formatting information, machine learning, ICML}

\vskip 0.3in
]

\begin{abstract}
Rejection sampling is a technique for sampling
from difficult distributions. However, its use is limited due to a high rejection rate.
Common adaptive rejection sampling
methods either work only for very specific distributions
or without performance guarantees. In this paper, we present
\emph{pliable rejection sampling} (\PRS), a new approach to rejection sampling, where we learn
the sampling proposal using a kernel estimator.
Since our method builds on rejection sampling,
the samples obtained are with high probability i.i.d.\@
  and distributed according to~$f$.
Moreover, \PRS comes with a guarantee
on the number of accepted samples.
\end{abstract}

\input{intro}
\vspace{0.2em}
\section{Setting}
\vspace{0.2em}
 Let $d \geq 1$ and let $f$ be a positive function with finite integral defined on $[0,A]^d$ where $A>0$ (we provide an extension to density defined on $\mathbb R^d$ itself in the Appendix~\ref{app:unb}), that we will call the \textit{target density}. Our objective is to provide an algorithm that samples from a normalized version of $f$ with a minimal number of requests to $f$, where a request is the evaluation of $f$ in a given point of choice. More precisely, the question we ask is the following.

\smallskip
\noindent
\begin{quote}
Given a number $n$ of requests to $f$, what is the number $T$ of samples $Y_1, \dots, Y_T$ that one can generate such that they are i.i.d.~and sampled according to $f$?
\end{quote}

\subsection{Assumption on the target density}
We make the following assumptions about~$f$.
\begin{assumption}[Assumption on the density]\label{ass:dens2}
The positive function $f$, defined on $[0,A]^d$ is bounded i.e.,~there exists $c> 0$ such that the density $f$ satisfies
$f(x) \leq c .$
Moreover,~$f$ can be uniformly expanded by a Taylor expansion in any point up to some degree $0<s\leq 2$, i.e.,~there exists $c'' >0$ such that, for any $x\in \mathbb R^d$, and for any $u\in \mathbb R^d$, we have
\[\left|f(x + u) - f(x) - \langle \bigtriangledown f(x), u\rangle \mathbf 1\{s >1\}\right| \leq c'' \|u\|_2^{s}.\]
\end{assumption}
For this assumption, we impose that $f$ is defined on $[0,A]^d$, but this could be relaxed to hold for any other convex compact of $\mathbb R^d$. For an alternative method and alternative assumptions that do not assume that $f$ has a bounded support, see Appendix~\ref{app:unb}, where this approach is described in detail.

Note that for this assumption, we do not impose that $f$ is a density: it must be a positive function, but it can be a non-normalized density (its integral may not be equal to $1$). This remark is particularly useful for Bayesian methods. The assumption also imposes that~$f$ is in a H\"older ball of smoothness $s$.
Notice that this is not very restrictive, in particular for a the case with small~$s$.

\subsection{Assumption on the kernel}
Let $K_0$ be a positive univariate density kernel defined on $\mathbb R$ and let
\[K = \prod_{i=1}^d K_0\]
be the $d$-dimensional product kernel associated with~$K_0$. This kernel will be used in the rest of the paper for interpolating $f$ using collected samples.
In order to be able to \textit{sample} from this kernel estimate, it would be more convenient to consider a kernel that corresponds to a density on $\mathbb R^d$ (hence a non-negative kernel) from which sampling is easier.
 A typical example of a useful kernel is then the Gaussian kernel:
$$K_0(x) = \frac{1}{\sqrt{2\pi}} \exp\left(-x^2/2\right).$$
Let us already mention that the Gaussian kernel satisfies also the prerequisites of the next assumption.
\begin{assumption}[Assumption on the kernel]\label{ass:ker}
The kernel $K_0$ defined on $\mathbb R$ is uniformly bounded i.e.,
$K_0(x) \leq C,$
and it is a density kernel, i.e.,~it is non-negative and $\int_{\mathbb R^d} K(x) dx = 1$.
%
Furthermore, it is also of degree $2$, i.e.,~it satisfies
$$\int_{\mathbb R} x K_0(x) dx = 0,$$
and, for some $C'>0$ 
$$\int_{\mathbb R} x^2 K_0(x) dx \leq C'.$$

Also, $K_0$ is $\varepsilon$-H\"older for some $\varepsilon>0$, i.e.~ $\exists~ C''>0$ such that for any $(x,y)\in \mathbb R^2$,
$$\left|K_0(y) - K_0(x)\right| \leq C'' \left|x-y\right|^{\varepsilon}.$$
\end{assumption}

For the Gaussian kernel, the above assumption holds with $C = 1$, $C'=1$, $C''=4$, and $\varepsilon = 1$. In our work, we mainly focus on the case of Gaussian kernel, since it is easy to sample from the resulting estimate, as it is a mixture of Gaussian distributions.
\vspace{0.2em}
\section{Algorithm and results}\label{algo_results}
\vspace{0.2em}
We first present the main tool which is a kernel estimator and a uniform bound on its performance. We then use it to describe our algorithm, \textit{\textbf{p}liable \textbf{r}ejection \textbf{s}ampling}  (\PRS).  We call our sampler \emph{pliable}, since it builds a proposal
by \emph{bending} the original uniform distribution.
Moreover, we provide a guarantee on its performance and present extensions to high dimensional situations.

%
%
%
\begin{figure}
\begin{center}
\begin{tikzpicture}
\begin{axis}[
  no markers, domain=0:9, samples=500,
  axis lines*=left, 
  height=5cm, width=8cm,
  xtick=\empty, ytick=\empty,
  enlargelimits=false, clip=false, axis on top,
  grid = major
  ]
  \addplot [thick, fill=gray!15, draw=none]coordinates {
            (9, 0.155)
            (9, 0)
            (0, 0)
            (0, 0.155)  };
  \addplot [very thick, red!50!black, fill=gray!15] {gauss(6,.45)+gauss(3,.95)+ 0.15}  node [pos=0.3, above left] {\scriptsize{envelope $\widehat M \widehat g$}};
  \addplot [mark =none, draw=none] {gauss(5,2.25)*5.3};
  \addplot [very thick, cyan!50!black, fill=white] {gauss(6,.5)+gauss(3,1)}
  node [pos=0.78, below left] {\scriptsize{target $f$}};
  \addplot [very thick,red!50!orange, dashed] {gauss(6,.45)+gauss(3.06,0.95)}
  node [pos=0.47, below left] {\scriptsize{estimate $\widehat{f}$}};
  \addplot[thick, samples=50, smooth, green!50!black, name path=three] coordinates {(2.7,0)(2.7,0.385)};
  \addplot[thick, samples=50, smooth, red, name path=four] coordinates {(2.7,0.39)(2.7,0.55)};
  \pgfplotsset{
    after end axis/.code={
      \node[black] at (axis cs: 2.7,-0.04){\small{$x$}};
      \node[black] at (axis cs: 3.3,0.465){\small{$\mathcal{R}$}};
      \node[black] at (axis cs: 6,0.4){\small{$\mathcal{A}$}};
        }
  }
\end{axis}
\end{tikzpicture}
\end{center}
\caption{Pliable rejection sampling}
\label{fig:PRS}
\end{figure}
\subsection{Uniform bounds for kernel regression estimation}\label{ss:estf2}
Let $X_1, \ldots, X_N$ be $N$ points generated uniformly on $[0,A]^d$. Let us define $h \eqdef h_s(\delta) = \left(\log(NA/\delta)/N\right)^{\frac{1}{2s+1}}$,
\begin{equation}\label{eqn:hatf2}
\widehat f(x) = \frac{A^d}{Nh^d} \sum_{k=1}^N f(X_i) K\left(\frac{X_i - x}{h}\right).
\end{equation}

\begin{theorem}[proved in Appendix~\ref{app:proof}]\label{thm:dimd2}
Assume that Assumptions~\ref{ass:dens2} and~\ref{ass:ker} hold with $0<s \leq 2$, $C,C',C'',c,c''>0$, and $\varepsilon >0$. The estimate $\widehat f$ is such that with probability larger than $1-\delta$, for any point $x \in [0,A]^d$,
\begin{align*}
\left|\widehat f(x) -  f(x)\right| 
\leq H_0 \left(\frac{\log(NAd/\delta)}{N}\right)^{\frac{s}{2s+d}},
\end{align*}
where $v = \log\left(1 + \frac{1}{c''+c}\right)\frac{2}{\min(1,s)}+ \frac{3}{\varepsilon} \log\left(1 + \frac{1}{C''c}\right)$ and $H_0$ is a constant that depends on $d, v, c, c'', C, C'$, and $A$.
\end{theorem}


\subsection{Pliable rejection sampling}
\PRS (Figure~\ref{fig:PRS}, Algorithm~\ref{alg:2}) aims at sampling as many i.i.d.\@ points distributed according to $f$ as we can with as little computations of $f$ as possible. It consists of three  steps:
\begin{enumerate}
\item In the beginning \PRS  samples the domain uniformly at random on $[0,A]^d$ for a number of $N$ samples and computes $f$ for these samples.
\item Then, \PRS uses these samples to estimate $f$ by a kernel regression method.
\item Finally, \PRS uses the newly obtained estimate plus the uniform bound on it, as a \textit{compact pliable proposal} for rejection sampling.
\end{enumerate}

Since this pliable proposal is close to the target density, the rejection sampling will reject only a small number of points by using it.
In \PRS, we set the constant
\[N \eqdef n^{\frac{2s+d}{3s+d}},\]
where $N$ is the number of evaluations of the function $f$ needed for the first estimation step that optimizes the number of accepted samples in the second step.
We also define
\[r_{N} \eqdef A^d H_C \left(\frac{\log(NAd/\delta)}{N}\right)^{\frac{s}{2s+d}},\]
where $H_C$ is a parameter of the algorithm.
Our method samples most of the samples by rejection sampling according to a \textit{pliable proposal} that is defined as
\begin{align}\label{eq:compprop}
\widehat g^\star \eqdef \frac{1}{\frac{A^d}{N} \sum_{i=1}^N f(X_i) + r_{N} } \left(\widehat f + r_{ N} \mathcal U_{[0,A]^d}\right),
\end{align}

where $\mathcal U_{[0,A]^d}$ is the uniform distribution on $[0,A]^d$, and $\widehat f$ is the estimate of $f$ defined in \eqref{eqn:hatf2} computed with the $N$ samples collected in the \emph{initial sampling} phase of \PRS. We also define the empirical rejection sampling constant as
\[\widehat M \eqdef \frac{A^d/N\sum_i f(X_i)+r_N}{A^d/N\sum_i f(X_i) - 5r_N}.\]


\begin{algorithm}[t]
\begin{algorithmic}
\caption{Pliable rejection sampling (\PRS)}
 \label{alg:2}
\STATE {\bf Parameters:} $s$, $n$, $\delta$, $H_C$

\STATE \textbf{Initial sampling}

\STATE \quad
\begin{tabular}{l}
Draw uniformly at random $N$ samples on $[0,A]^d$ \\ \quad and evaluate $f$ on them
\end{tabular}
\STATE \textbf{Estimation of $f$}

\STATE \quad
\begin{tabular}{l}
Estimate $f$ by $\widehat f$ on these $N$ samples (Section~\ref{ss:estf2})
\end{tabular}
\STATE \textbf{Generating the samples}

\STATE \quad
\begin{tabular}{l}
Sample $n- N$ samples from \\ \quad  the \textit{compact pliable proposal} $\widehat g^\star$
\end{tabular}
\STATE \quad
\begin{tabular}{l}
Perform rejection sampling on these samples \\ \quad using $\widehat M$ as a rejection constant to get $\widehat n$ samples
\end{tabular}
\STATE \textbf{Output:} Return the $\widehat n$ samples
\end{algorithmic}
\end{algorithm}
\subsection{Analysis}
In the following, we state and prove the main result about \PRS, which is a guarantee on the number of samples.

\begin{theorem}\label{thm:prs2}
Assume that Assumptions~\ref{ass:dens2}, and~\ref{ass:ker} hold with $0<s \leq 2$, and that $H_C$ is an upper bound on the constant $H_0$ from Theorem~\ref{thm:dimd2} (applied to $f$ and $\widehat f$), and that $$8r_N\leq \int_{[0,A]^d} f(x)dx.$$
Then with probability larger than $1-\delta$, the samples are generated as i.i.d.\@ according to $f$ and  for $n$ large enough,
the number $\widehat n$ of samples generated is at least
$$\widehat n \geq n\left[1 -  \cO\left(\frac{\log(nAd/\delta)}{n}\right)^{\frac{s}{3s+d}}\right].$$
\end{theorem}

\begin{proof}
By Theorem~\ref{thm:dimd2} and the definition of $r_N$, we have that with probability larger than $1-\delta$, for any $x \in [0,A]^d$,
$$\left|\widehat f(x) - f(x)\right| \leq r_N \frac{1}{A^d} =  r_N \mathcal U_{[0,A]^d}.$$
Let $\xi'$ be the event where the above holds. It has probability larger than $1-\delta$. 
Now let us define event $\xi''$ as
\begin{align*}
\xi'' &\eqdef \Bigg\{\left|\frac{A^d}{n} \sum_{i=1}^n f(X_i) - \int_{[0,A]^d} f(x)dx\right|\\
&\qquad \qquad \leq 2A^dc\sqrt{\frac{1}{N} \log(1/\delta)} \eqdef c_N\Bigg\}.
\end{align*}
By Hoeffding's inequality, we know that the probability of $\xi''$ is larger than $1-\delta$.
Let
$\xi = \xi' \cap \xi'',$
the probability of $\xi$ is larger than $1-2\delta$.
Therefore, we have that on $\xi$, 
\begin{align*}
\widehat g^\star &= \frac{\widehat f + r_{ N}\mathcal{U}_{[0,A]^d} }{ A^d/n \sum_{i=1}^n f(X_i) + r_N}\\
&\geq \frac{f}{ \int_{[0,A]^d} f(x)dx + r_N + c_N}\\
&\geq \frac{f}{ \int_{[0,A]^d} f(x)dx} (1  - 4r_N/m),
\end{align*}
with $m =  \int_{[0,A]^d} f(x)dx$ and where we used that $$m \geq 8r_N \geq  4r_N + 4c_N.$$
Note that on $\xi$
\begin{align*}
\frac{1}{1  - 4r_N/m} &= \frac{m}{m  - 4r_N}\\
&\leq \frac{A^d/N\sum_i f(X_i)+c_N}{A^d/N\sum_i f(X_i) - c_n  - 4r_N}\\
&\leq \frac{A^d/N\sum_i f(X_i)+r_N}{A^d/N\sum_i f(X_i) - 5r_N} = \widehat M,
\end{align*}
so that on $\xi$, the rejection sampling constant $\widehat M$ is indeed appropriate. We also have on $\xi$,
\begin{align*}
\widehat M &= \frac{A^d/N\sum_i f(X_i)+r_N}{A^d/N\sum_i f(X_i) - 5r_N}\\
&\leq \frac{m+r_N + c_N}{m- 5r_N-c_N} \\
&\leq \frac{m+2r_N}{m- 6r_N}.
\end{align*}
Therefore, on $\xi$, the rejection sampling is going to provide samples that are i.i.d.~according to $f$, and $\widehat n$ will be a sum of Bernoulli random variables of parameter larger than
\begin{align*}
\frac{1}{\widehat M} &\geq \frac{m- 6r_N}{m+2r_N}\\
&\geq (1 - 6r_N/m)(1 - 4r_N/m) \\
&\geq 1 - 20 r_N/m,
\end{align*}
since $m \geq 8r_N$.
We have that on $\xi$, with probability larger than $1-\delta$,
\[\widehat n \geq (n-N)(1  - 20r_N/m) - 2\sqrt{n\log(1/\delta)}.\]
This implies, together with the definition of $r_N$, $\widehat n$ is with probability larger than $1-3\delta$ lower bounded as
\begin{align*}
\widehat n & \geq  (n -N ) \left(1 -  20r_N/m -  4\sqrt{\frac{\log(1/\delta)}{n}}\right).
\end{align*}
Since
$$N = n^{\frac{2s+d}{3s+d}},$$
we have that for $n$ large enough, with probability larger than $1-3\delta$, there exists a constant $K$ such that
\begin{align}
\widehat n \geq n\left[1 - K\log(nAd/\delta)^{\frac{s}{3s+d}}n^{-\frac{s}{3s+d}} \right].
\end{align}
\end{proof}
\vspace{-1em}
Theorem~\ref{thm:prs2} implies that the number of rejected samples is negligible when compared to $n$: Indeed, the number of rejected samples divided by $n$ is of order
\[\left(\frac{\log(nd/\delta)}{n}\right)^{\frac{s}{3s+d}}.\]
This statement shows a light-years difference between \SRS and \PRS.
Therefore, unlike in \SRS where we only accept a fraction of samples, here we asymptotically accept almost all the samples.

\subsection{Discussion}
\paragraph{Rejected samples.} Theorem~\ref{thm:prs2} states that if we have an admissible proposal density $g$ and associated upper bound, as well as a lower bound $s$ on the smoothness of  density~$f$, then with high probability, \PRS rejects (asymptotically) only a negligible number of samples with respect to~$n$: Almost one sample is generated for every unit of budget spent, i.e.,\@ one call of $f$. This implies in particular that our bounds in terms of a number of i.i.d.~samples generated according to $f$ per computation of $f$ are better than the ones for \Astar~\citep{maddison2014a}.

On the other hand, it is not easy to do a direct comparison with \MCMC methods since these methods generate correlated samples with stationary distribution $f$ (asymptotically) while we generate exact i.i.d.~samples generated according to $f$ (with high probability). However, for any sample generated by \MCMC, we need to call $f$ once anyway, which is asymptotically the same as for our strategy.

\paragraph{Sampling from the pliable target.} If, for instance, one takes a Gaussian kernel $K_0$, then sampling from the pliable proposal
\[\widehat g^\star = \frac{1}{A^d/N\sum_{i=1}^n f(X_i) + r_{ N} } \Big(\widehat f + A^d r_{ N} \mathcal U_{[0,A]^d} \Big)\]
is very easy, since it is a mixture of Gaussian distributions (in $\widehat f$, by definition of a kernel estimator), and a uniform distribution on $[0,A]^d$.


\paragraph{The condition on $\int_{[0,A]^d} f(x)dx$.} In Theorem~\ref{thm:prs2}, we need that $$\int_{[0,A]^d} f(x)dx\geq 8r_N,$$ so that the empirical rejection sampling $\widehat M$ is not too large. If $\int_{[0,A]^d} f(x)dx$ is very small, then it means that $f$ is very peaky and therefore extremely difficult to estimate, besides the trivial case where $f=0$. This assumption is not very constraining since $r_N$ converges to $0$ with $N$ and therefore also with $n$.

\paragraph{Normalized distribution.} If the distribution $f$ is normalized, i.e., $$\int_{[0,A]^d} f = 1,$$ then the algorithm can be simplified. Indeed, the \textit{pliable proposal} can be taken as the mixture
$$\frac{1}{1 + r_{N} } \left(\widehat f + r_{ N} \mathcal U_{[0,A]^d}\right),$$
removing the normalisation constant $A^d/N \sum_i f(X_i)$. In this case, instead of $\widehat M$, we can simply use $1+r_N$ as the rejection sampling constant.

\subsection{Case of a distribution with unbounded support} In the case where the distribution $f$ is not assumed to have bounded support, our method does not directly apply since it involves uniform sampling on the domain. One way to go around this, in the case where $f$ is sub-Gaussian, is to sample on uniformly not on $[0,A]^d$, but on a hypercube centered in $0$ and of side length $\sqrt{\log(n)}$, and then perform our method using this hypercube as the domain. Then, we would estimate $f$ as $0$ outside this hypercube. Because of the properties of sub-Gaussian distributions that have vanishing tails, this will provide results that are similar to the ones on $[0,A]^d$, but with $A$ replaced by $\sqrt{\log(n)}.$ Then, for instance, the bound in Theorem~\ref{thm:dimd2} provides a bound that would scale on $\mathbb R^d$ itself as
$$\left|\widehat f(x) -  f(x)\right| \leq \cO\left( \log(n)^{d/2}\left(\frac{\log(Nd/\delta)}{N}\right)^{\frac{s}{2s+d}}\right),$$
i.e.,\@ the bound would become worse by a factor $\log(n)^{d/2}$. This would imply that the bound of Theorem~\ref{thm:prs2} would also become worse by a factor of $\log(n)^{d/2}.$ This is not a problem when $d$ is very small. However, even in the case where~$d$  is moderately small, this becomes quickly a problem. For this reason, this may not necessarily be a good approach in all cases, for a density with an unbounded support. To deal with this case, a better idea is to do a two-step procedure of rejection sampling, and then estimate $f$ by density estimation instead of regression estimation. (See Appendix~\ref{app:unb} for more details.) In this way, we avoid the problem of paying this additional $\log(n)^d$ in the bound. The algorithm is however slightly more complicated.


\subsection{Extensions for high dimensional cases (large $d$)}
\label{ss:extd}

One known limitation of rejection sampling is its lack of scalability with the dimension $d$. While our methodology mainly applies to small dimensions,
we now discuss some modifications of the method in order to better handle some specific cases when the ambient dimension $d$ is large,
and leverage the scalability of the initial phase.
To this end, we resort to optimization techniques that enable to approximately localize the mass of the distribution in time at most quadratic in $d$ (and possibly $\sqrt{d}$), assuming that the density is convex on the region of small mass and arbitrary on the region of high mass:
\begin{definition}We define the $\gamma$-support $\text{Supp}_{f,\gamma}$ of $f$ as the closure of its $\gamma$-level set $\Lambda_{f,\gamma}$, that is
\[\text{Supp}_{f,\gamma} = \overline{\Lambda}_{f,\gamma}\quad \text{where}\quad\Lambda_{f,\gamma} \eqdef \left\{x\in\mathcal D :  f(x) > \gamma \right\}\,,\]
\end{definition}
We say one \textit{localizes} the $\gamma$-support of $f$ if it finds some $x\in\text{Supp}_{f,\gamma}$. This is however non-trivial:
\begin{lemma}
In the general case when no assumption is made on $f$, localizing the $0$-support of $f$ may take a number of evaluation points exponential in the dimension $d$.
\end{lemma}
\vspace{-4mm}
\begin{proof} Indeed, using uniform sampling, this requires at least $|\mathcal D|/|\text{Supp}_{f,0}|$ samples on average. If we introduce $R$ such that $\mathcal D$ has the same volume as the Euclidean ball  of radius $R$ centered at $0$, $B_d(R) \subset \Real^d$,
and similarly $r_0$ such that $|\text{Supp}_{f,0}| = |B_d(r_0)|$, this means we need $(R/r_0)^d$ samples on average.
\end{proof}
\vspace{-4mm}
Thus, without further structure, the initial sampling phase of Algorithm~\ref{alg:2}  may require exponentially many steps.
We thus consider a more specific situation.
In practice, for numerical stability, it is important to be able to sample points that are not only in $\text{Supp}_{f,0}$ but also in $\text{Supp}_{f,\gamma}$, for $\gamma>0$ away from $0$.
Let $r_\gamma$ be such that $|\text{Supp}_{f,\gamma}| = |B_d(r_\gamma)|$.
We assume that $\text{Supp}_{f,0} = \mathcal D$ (and thus $r_0=R$) but
$R/r_\gamma = c_\gamma>1$, where $c_\gamma$ is not small, say $c_\gamma\geq 2$, which models a situation when
it is easy to localize the $0$-support but a priori hard to localize the $\gamma$-support.

Now, we assume that the restriction of $f$ on the complement of its  $\gamma$-support, $f_{|\text{Supp}_{f,\gamma}^c}$ is convex. This situation captures practical situations when the mass of the distribution is localized in a few small subsets of $\Real^d$. Note that $f$ does not need to be convex on $\text{Supp}_{f,\gamma}$ and that $\text{Supp}_{f,\gamma}$ can consist of several disjoint connected sets; thus $f$ does not need to be unimodal.
\begin{lemma}
Under the previous assumptions, if  we
can additionally evaluate $f$ and its gradient point-wise, it is possible to
find a solution $x$ in $\text{Supp}_{f,\gamma-\varepsilon}$ in no more than $O(d^2/\varepsilon^2)$, that is to localize $\text{Supp}_{f,\gamma}$ in less than an exponential number (with $d$) of trials, by replacing the uniform sampling scheme in the \emph{initial sampling} phase of \PRS with a combination of uniform sampling and convex optimization techniques.
\end{lemma}
\vspace{-4mm}
\begin{proof}
The proof is as follows. First, since $r_0=R$, we can find a point $x_0$ in $\text{Supp}_{f,\gamma}^c$ in $\cO(1)$ trials by uniform sampling. Now, since $f_{|\text{Supp}_{f,\gamma}^c}$ is convex, it is maximal on the boundary of its domain, that is, on $\partial\text{Supp}_{f,\gamma}$. Thus, we use standard optimization techniques to find the maximum of $f$, starting from $x_0$: Using the fact that $f$ and its gradient can be evaluated point-wise, the simplest gradient descent scheme  (see e.g., \citealp{Nesterov2004}, Theorem 3.2.2 with parameter 3.2.10) finds a solution $x$ in $\text{Supp}_{f,\gamma-\varepsilon}$ in no more than
$\cO(d^2/\varepsilon^2)$ evaluation steps.
\end{proof}
\vspace{-4mm}
Note that using more refined (but more computationally and memory-wise expensive) methods such as the one from (\citealp{Nesterov2004}, 4.2.5 p.187) that relies on point-wise evaluations of the Hessian, one can get a solution $x\in\text{Supp}_{f,\gamma-\varepsilon}$ in no more than $\cO(\sqrt{\nu}\ln(\nu/\varepsilon))$ steps, assuming we can build a $\nu$-self concordant barrier function (see \citealp{Nesterov2004} for more explanations regarding such functions). As we are able to build $\cO(d)$-self-concordant barrier (and even $(1+\smallO(1))d$-self-concordant barrier, see \citealp{bubeck2015entropic,hildebrand2014canonical}, but at the price of a possibly high computational cost), it is then possible to get a solution in only $\cO(\sqrt{d}\ln(d/\varepsilon))$ steps. This is another example where one can get an exponential improvement over the general situation.

Now, repeating this procedure $T_s$ times (sample a starting point uniformly at random in $\mathcal D$, then optimize$f$ from this starting point), we can get $T_s$ evaluation points in $\text{Supp}_{f,\gamma-\varepsilon}$ in only $\cO(T_sd^2/\varepsilon^2)$ and respectively $\cO(T_s\sqrt{d}\log(d/\varepsilon))$ steps.

Finally, this naturally extends to cases when  $f_{|\text{Supp}_{f,\gamma}^c}$ may not be convex, but
$T(f)_{|\text{Supp}_{f,\gamma}^c}$ is convex for some known transformation $T$, and that $T(f)$, its gradient and its Hessian can all be evaluated point-wise: This is useful in particular when $f$ can only be evaluated up to a normalization constant, as is the case here and often in practice.
\input{Experiments}
\input{conc}

{\small
\textbf{Acknowledgements}
We thank Chris Maddison for his code of \Astar.
The research presented in this paper was supported by CPER Nord-Pas de Calais/FEDER DATA Advanced data science and technologies 2015-2020, by French Ministry of
Higher Education and Research, Nord-Pas-de-Calais Regional Council,  French National Research Agency project ExTra-Learn (n.ANR-14-CE24-0010-01),
and by German Research Foundation's Emmy Noether grant MuSyAD (CA 1488/1-1).
}

\bibliography{library} 

\bibliographystyle{icml2016}

\clearpage

\onecolumn
\input{app}

\end{document}

%% file: misosymbols.tex
\newtheorem{assumption}{Assumption}
\newtheorem{lemma}{Lemma}
\newtheorem{theorem}{Theorem}
\newtheorem{definition}{Definition}

\newcommand{\EE}[1]{\mathbb{E}\left[#1\right]}

\newcommand*{\MyDef}{\mathrm{\tiny def}}
\newcommand*{\eqdefU}{\ensuremath{\mathop{\overset{\MyDef}{=}}}}
\newcommand*{\eqdef}{\mathop{\overset{\MyDef}{\resizebox{\widthof{\eqdefU}}{\heightof{=}}{=}}}}

\newcommand{\cO}{\mathcal{O}}
\newcommand\smallO{
  \mathchoice
    {{\scriptstyle\mathcal{O}}}
    {{\scriptstyle\mathcal{O}}}
    {{\scriptscriptstyle\mathcal{O}}}
    {\scalebox{.7}{$\scriptscriptstyle\mathcal{O}$}}
  }

\newcommand{\cR}{\mathcal{R}}

\newcommand{\nothere}[1]{}

\usepackage{xspace}

\newcommand{\PRS}{\texttt{\textcolor[rgb]{0.5,0.2,0}{PRS}}\xspace}
\newcommand{\EPRS}{\texttt{\textcolor[rgb]{0.5,0.2,0}{EPRS}}\xspace}

\newcommand{\SRS}{\texttt{SRS}\xspace}
\newcommand{\MH}{\texttt{MH}\xspace}
\newcommand{\MCMC}{\texttt{MCMC}\xspace}
\newcommand{\ARS}{\texttt{ARS}\xspace}
\newcommand{\Astar}{\texttt{A$^\star$\,sampling}\xspace}
\newcommand{\OSstar}{\texttt{OS$^\star$}\xspace}

%% file: intro.tex
\section{Introduction}
\vspace{0.2em}
In machine learning, we often need to sample from distributions. \emph{Rejection sampling}  is a known textbook method for sampling from density $f$ with intractable direct sampling. The basic method (\SRS, Figure~\ref{fig:SRS}) constructs an \emph{envelope}  $Mg$ that is an upper bound on $f$, where $g$ is a \emph{proposal distribution} from which we can sample easily.
Each time we get a sample from $g$, we accept or reject it with probability depending on the value of $g$ and $f$ in this point. 
To guarantee efficiency, a good proposal distribution is a necessary knowledge
we need to provide to the sampler.
In the absence of such knowledge, we typically resort to a uniform upper bound
on $f$ which results in high rejection rates and the method stays in textbooks. What's wrong with a high rejection rate?  The reason is that for every point proposed, we need to call~$f$  to decide whether this point is accepted. If many points are rejected, then $f$ is called many times with few generated samples. When evaluating $f$ is costly, then we are wasting resources.


\begin{figure}
\vspace{-0.1em}
\begin{center}
\begin{tikzpicture}
\begin{axis}[
  no markers, domain=0:9.5, samples=500,
  axis lines*=left, 
  height=5cm, width=8cm,
  xtick=\empty, ytick=\empty,
  enlargelimits=false, clip=false, axis on top,
  grid = major
  ]

  \addplot [mark =none, fill = gray!15, draw=none] {gauss(5,2.25)*5.3} \closedcycle;
  \addplot [mark =none, fill = white, draw=none] {gauss(6,.5)+gauss(3,1)};
  \addplot [very thick,cyan!50!black, name path = plot1] {gauss(6,.5)+gauss(3,1)}
  node [pos=0.74, below left] {\scriptsize{target $f$}};
  \addplot [very thick,red!50!black, name path = plot2] {gauss(5,2.25)*5.3}
  node [pos=0.4, above left] {\scriptsize{envelope  $M g$}};
  \addplot [very thick,red!50!orange, dashed] {gauss(5,2.25)*3}
  node [pos=0.58, above left] {\scriptsize{proposal  $g$}};
  \addplot[thick, samples=50, smooth,green!50!black, name path=three] coordinates {(2.8,0)(2.8,0.39)};
  \addplot[thick, samples=50, smooth,red, name path=four] coordinates {(2.8,0.395)(2.8,0.58)};
    \pgfplotsset{
      after end axis/.code={
        \node[black] at (axis cs:2.8,-0.04){\small{$x$}};
        \node[black] at (axis cs: 5,0.79){\small{$\mathcal{R}$}};
        \node[black] at (axis cs: 3.6,0.12){\small{$\mathcal{A}$}};
          }
    }
\end{axis}
\end{tikzpicture}
\vspace{-0.6cm}
\end{center}
\caption{Simple rejection sampling}
\label{fig:SRS}
\vspace{-0.5cm}
\end{figure}

\footnotetext{Research done during author's stay at SequeL, INRIA Lille.}

To alleviate this problem, \emph{adaptive} rejection samplers~\citep{gilks1992derivative,gilks1992adaptive,martino2011generalization} 
increase the acceptance rate by taking advantage of particular properties of~$f$. They construct
a proposal $g$ that is better adapted to $f$ than just a uniform distribution.
Adaptive rejection sampling (\ARS, \citealp{gilks1992adaptive}) is the most known among them.  \ARS works when the target is \emph{log-concave} and constructs a sequence of proposal densities tailored to~$f$.
In particular, if a sample that is drawn from a proposal $g_{t}(x)$ is rejected, this sample is used to build an improved proposal, $g_{t+1}(x)$, with a higher  acceptance rate.
\ARS then adds the rejected point to set $S$ of points defining an envelope of $f$ in order to decrease the area $\cR$ between the proposal and the target density~\citep{gilks1992adaptive,gilks1992derivative}. However, \ARS  can only be applied for log-concave (and thus \emph{unimodal}) densities, which is a stringent constraint in practice~\citep{gilks1992adaptive,martino2011generalization}
and therefore its use is limited.

The adaptive rejection Metropolis sampling \citep{gilks1995adaptive} extends \ARS to deal with non-log-concave densities by adding a Metropolis-Hastings (\MH) control step after each accepted sample.
However, the algorithm produces a Markov chain where the resulting samples are correlated.
Another adaptive method  is the convex-concave adaptive rejection sampling~\citep{gorur2011concave} where the target distribution is decomposed as the sum of convex and concave functions. In this method, the concave part is treated as in \ARS and uses the same set $S$ to construct an upper bound for the convex part by considering the secant lines.

A recent approach is \Astar~\citep{maddison2014a} that
was built on generalizing the
Gumbel-Max trick to the continuous case.
This method allows to sample from $f(x) \propto \exp\left(\phi(x) \right)$, where $\phi(x) = i(x) + o(x)$ for some bounded $o(x)$, and some tractable $i(x)$ that is equivalent to the proposal of a classical rejection sampling method.
We will relate to \Astar and compare to it empirically as well.
A similar approach is done in \OSstar~\citep{dymetman2012os} where the sampling is done according to the volume of the region under the proposal.

All these adaptive rejection sampling methods either pose strong assumptions on $f$
or do not come with performance guarantees. In this paper, we give an \emph{adaptive strategy}
that can work for a general class of densities and \emph{guarantee the number of accepted samples}. An interesting approach for the related, yet different problem of adaptive importance sampling, can be found in the work of \citet{zhang1996nonparametric}, where the author aims at integrating a function according to a density. To be  efficient and sequential, \citeauthor{zhang1996nonparametric} sequentially approximates the density times the absolute value of the function to be integrated by kernel methods and sample from this approximation. In particular, \citeauthor{zhang1996nonparametric} estimates the integral of interest by a weighted sum of the collected samples, where the weights depend on the distance between the estimated product function and the true product. This method is interesting because it is non-parametric and therefore requires few assumptions about the shape of the target object. 

In this paper, we consider a related idea for rejection sampling. In particular, we use \emph{non-parametric kernel methods} to estimate the target density. This estimate is then used to build a proposal density from which samples are drawn in order to improve the acceptance. This idea is related to the results of~\citet{zhang1996nonparametric} but there is a significant difference coming from a difference between importance sampling and rejection sampling and which makes the rejection sampling problem harder: While importance sampling requires only a proposal estimator that is good according to the $L_2$ risk, rejection sampling requires a proposal estimator that is good in $L_{\infty}$ risk and with high probability. This highlights a fundamental difference between rejection sampling and importance sampling and makes the problem of adaptive rejection sampling significantly more challenging than the problem of adaptive importance sampling.

To address this challenge, we present \emph{pliable rejection sampling} (\PRS), a simple variation of rejection sampling. Based on recent advances in density estimation and associated confidence sets, which allow to obtain a \textit{uniform} bound on the estimation error of estimators~\citep{tsybakov1998pointwise, korostelev1999asymptotic, gine2010adaptive, gine2010confidence} we propose a method where the proposal is an upper bound on the density that is based on a \emph{kernel estimator} of the density. 
The motivation behind the choice of a kernel estimator
comes (i) from the guarantees on the quality of the estimate and (ii) from the
ability to easily sample from it for some specific kernels.

\PRS has several advantages. First, it does not pose strong assumptions on~$f$
and assumes only mild smoothness properties. For instance, our assumptions are weaker than existing assumptions like log-concavity, concavity or convexity, since if a function satisfies any of these assumptions, then it is in a Besov ball of smoothness two, and therefore smooth enough for our method. Second, it is easy to implement, since
it combines common kernel density estimation and traditional rejection sampling.
Finally, it comes with a clean and tractable analysis which provides guarantees on the
number of samples for a given number of calls to~$f$. Our results imply that asymptotically, if we have a budget of $n$ calls to $f$,
then with high probability, we will obtain~$n$ i.i.d~samples distributed according to~$f$ up to a negligible term. Our procedure is therefore asymptotically almost as efficient as if we were sampling according to~$f$ itself.

\PRS is actually more efficient than \Astar in the sense of \emph{budget}. Indeed, in order to generate a single sample from $f$ using \Astar, we need to consume several calls to $f$. This implies that even in the asymptotic regime if we have a budget of $n$ calls to $f$, we will obtain less  than $a\times n$ i.i.d~samples distributed according to $f$  where $a<1$ is a small constant.
Furthermore, a \emph{huge difference} between \PRS and \Astar,
that makes \PRS practically appealing, is that the user
\emph{does not need to provide} any major information such as
a \emph{decomposition} of $f$ into $i(x)$ and $o(x)$ as in the case of  \Astar.

Since \PRS is based on rejection sampling, it is useful in the case
when the sample space is low-dimensional and when~$f$ is not very \emph{peaky},
which is also the case for \Astar. In particular, \Astar needs to find the maximum of the convolution of $f$ and a Gumbel process,  in order to output a sample. A typical case of a peaky distribution is a posterior
distribution commonly present in Bayesian approaches, where computationally efficient \MCMC methods~\citep{metropolis1949monte,andrieu2003introduction}
are the tool of choice, as they also scale much better with the dimension.
However, the samples are typically correlated and additional
measures need to be taken to make the samples perfect~\citep{andrieu2003introduction}.
Even though there exist \MCMC methods that perform perfect
sampling~\citep{propp1998coupling,fill1998interruptible}, they have to assume certain restrictions,
and are not used in practice since they are not efficient~\citep{andrieu2003introduction}.

In contrast, our method is a \emph{perfect sampler with high probability}. Our analysis shows that with high probability, asymptotically, each computation of~$f$ leads to the sampling of an i.i.d.\@ sample according to $f$. In Section~\ref{ss:extd},
we also provide an extension on how to deal with the high dimensional case and the case of a peaky density (as in the Bayesian posterior case) by a localization method.

%% file: Experiments.tex
\section{Numerical experiments}
\begin{figure*}[t]
\vspace{-0.2em}
\centering
        \begin{subfigure}[b]{0.32\textwidth}
                \includegraphics[width=\textwidth]{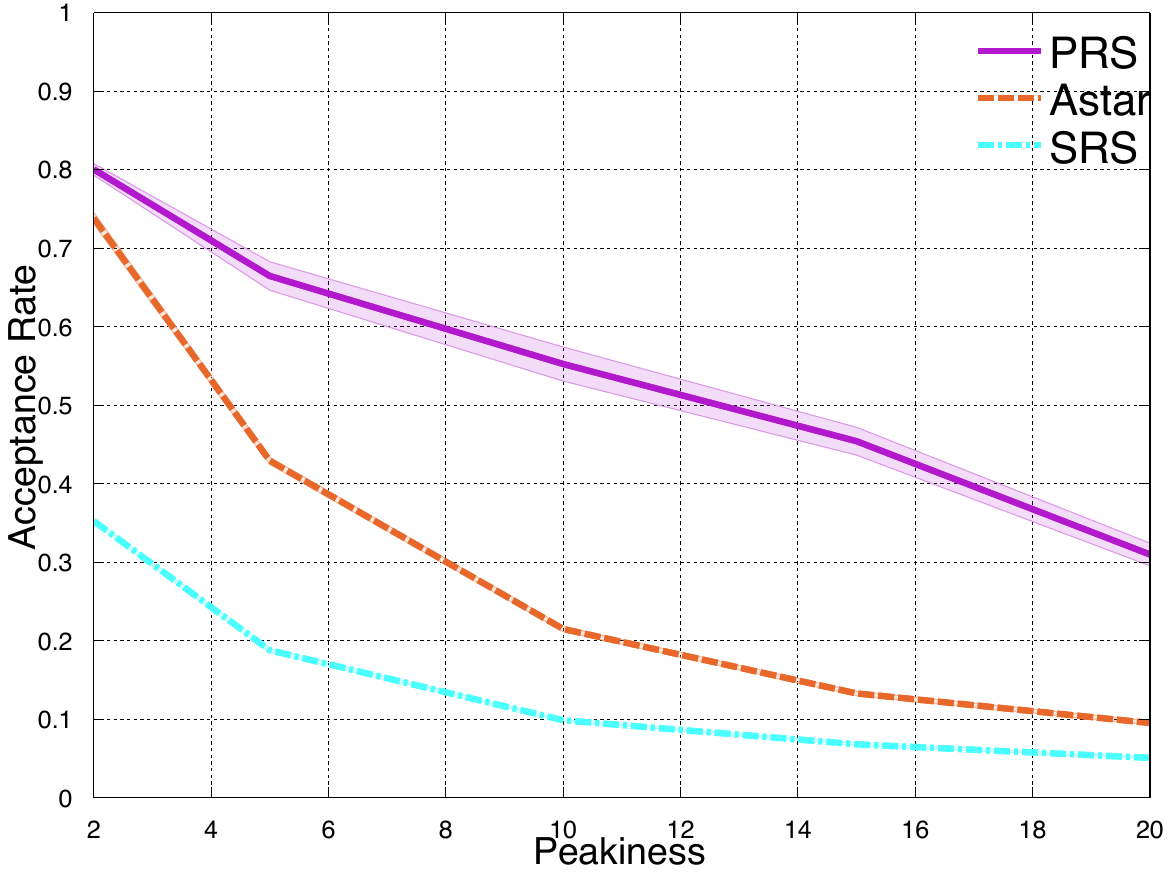}
                \caption{$n=10^4$}
                \label{a}
        \end{subfigure}
        ~
        \begin{subfigure}[b]{0.32\textwidth}
                \includegraphics[width=\textwidth]{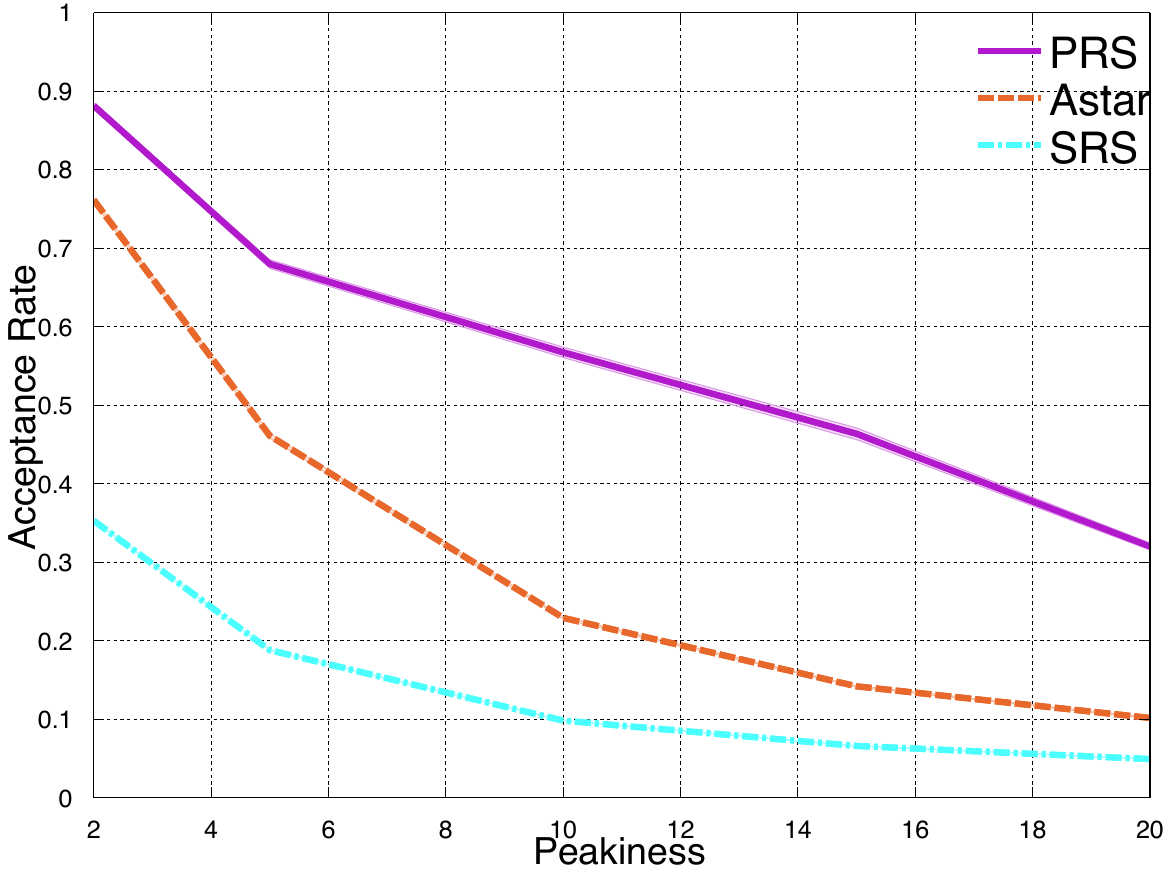}
                \caption{$n=10^5$}
                \label{b}
        \end{subfigure}
        ~
                \begin{subfigure}[b]{0.32\textwidth}
                \includegraphics[width=\textwidth]{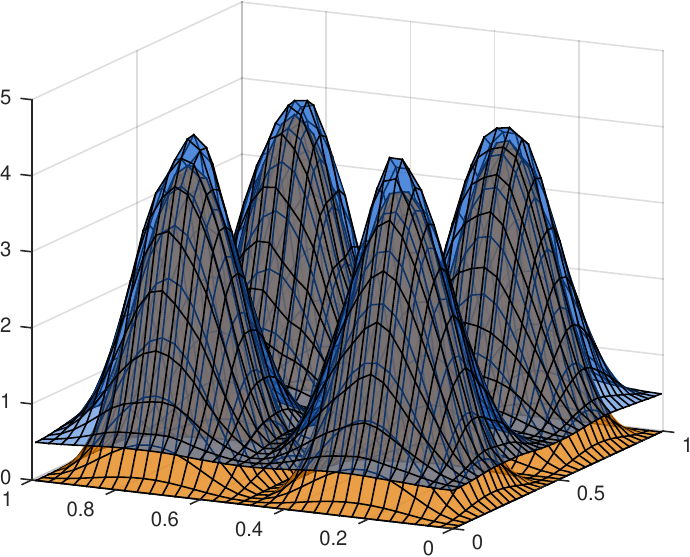}
                \caption{2D example}
        \end{subfigure}

        \caption{\textbf{Left, center:} Acceptance \textit{vs.}~peakiness. \textbf{Right:} 2D target (orange) and the pliable proposal (blue)}
        \label{fig:Acc_Peak}
        \vspace{-0.5em}
\end{figure*}

We compared \PRS to \SRS and \Astar numerically.
In particular, we evaluated the sampling rate, i.e., the
proportion of samples that a method gives with respect to the number of evaluation of $f$.
This is equal to the definition of \emph{acceptance rate} for \SRS and \PRS.

All the experiments were run with $\delta = 0.01$. $H_C$ was set through a cross-validation in order to provide a good proposal quality, i.e., how close is the proposal to the target distribution. $H_C$ is a problem dependent quantity and can capture prior information on the smoothness of $f$.

The goal of our experiments is to (i) show that \PRS outperforms \SRS with the same amount of evaluations of $f$
and (ii) that \PRS's performance is comparable to \Astar, which is a recent state-of-the-art sampler.
We use two of the same settings as in~\citet{maddison2014a}.
We emphasize again that \Astar \emph{is given extra
information} in form of the decomposition,  $f(x) \propto \exp\left(i(x) + o(x) \right)$ that \PRS does not need
and that is not available in general.
\subsection{Scaling with peakiness}
We first study the behavior of the acceptance rate as a function of the peakiness of~$f$.
In particular, we use the  target density of~\citet{maddison2014a},
$$ f(x) \propto \frac{e^{-x}}{(1+x)^a},$$
where $a$ is the \emph{peakiness parameter}.
By varying $a$, we can control the difficulty of accepting a sample coming from a proposal distribution.
For \Astar, we use the same decomposition of $\phi(x) = i(x) + o(x)$
as~\citeauthor{maddison2014a} .

Figure~\ref{fig:Acc_Peak} gives the acceptance rates of all these methods for $a \in \{2, 5, 10, 15, 20\}$ averaged over 10 trials. Figure~\ref{a} corresponds to a budget of $n=10^5$ requests to~$f$ and Figure~\ref{b} to a budget of $n=10^6$ requests.  \PRS performs in both cases better than the \Astar and \SRS. Moreover, the performance of \PRS improves with $n$.
Indeed, with a larger number of evaluations, the estimate $\widehat{f}$ gets better and more precise, allowing the construction of a tighter upper bound. This provides a good quality proposal that is, in this case, able to perform better than
\SRS and even outperform \Astar even for a low peakiness.

\subsection{Two-dimensional example}
In this part, we compare the three methods on the distribution defined on [0,1]$^2$ as
\[ f(x,y) \propto \left(1 + \sin\left(4\pi x-\frac{\pi}{2}\right)\right) \left(1 + \sin\left(4\pi y-\frac{\pi}{2}\right)\right).\] 
\noindent

Figure~\ref{fig:Acc_Peak} (right)
shows the target density, along with the derived envelope.
%
Table \ref{tab:2d} gives the acceptance rates of the three methods for $n=10^6$, where \PRS outperforms \SRS
and approaches the performance of \Astar.

\begin{table}[!htb]
        \begin{subtable}{\linewidth}
        \centering
        \begin{tabular}{ c | c c }
        $n = 10^6$&  \makecell{\emph{acceptance}  \emph{rate}}& \makecell{\emph{standard}  \emph{deviation}}\\ \hline
        \textbf{\PRS }& 66.4\% & 0.45\% \\
        \textbf{\Astar }& 76.1\% & 0.80\% \\
         \textbf{\SRS} & 25.0\% & 0.01\%
        \end{tabular}
        \end{subtable}
        \caption{2D example:\,Acceptance rates  averaged over 10 trials}
        \label{tab:2d}
\end{table}

\subsection{Clutter problem}
In order to illustrate how \PRS behaves for inference tasks, we tested the methods on the clutter problem of  \citet{minka2001expectation} as did~\citet{maddison2014a}.
The goal is to sample from the posterior distribution of the mean of normally distributed data with a fixed isotropic covariance, under the assumption that some points are outliers. The setting is again the same as the one of \citet{maddison2014a}: In $d$ dimensions, we generate 20 data points, a half from $[-5,-3]^d$ and another half from $[2,4]^d$, which provides a bimodal posterior that is very peaky.

Table \ref{tab:clutter} gives the acceptance rates for the clutter problem in the 1D and 2D cases with a budget of $n=10^5$ requests to the target $f$. This target is the posterior distribution of the mean. In this case, even if \PRS gives a reasonable acceptance rate, it is not performing better than \Astar.
\begin{table}[!htb]
        \begin{subtable}{\linewidth}
               \centering\hspace{-0.2cm}
        \begin{tabular}{ c | c c }
        $n = 10^5$, 1D &  \makecell{\emph{acceptance}  \emph{rate}} &\makecell{\emph{standard}  \emph{deviation}}\\ \hline
        \textbf{\PRS }& 79.5\% & 0.2\% \\
        \textbf{ \Astar }& 89.4\% & 0.8\% \\
         \textbf{\SRS} & 17.6\% & 0.1\%
        \end{tabular}
        \end{subtable}
\\
\vspace{0.2cm}

        \begin{subtable}{\linewidth}
               \centering
        \begin{tabular}{ c | c c }
       $n = 10^5$, 2D&  \makecell{\emph{acceptance}  \emph{rate}} & \makecell{\emph{standard}  \emph{deviation}}\\ \hline
       \textbf{\PRS }& 51.0\% & 0.4\% \\
       \textbf{\Astar }& 56.1\% & 0.5\% \\
        \textbf{\SRS} & $2\cdot 10^{-3}$\% & $10^{-5}$\%
        \end{tabular}
        \end{subtable}
        \caption{Clutter problem:\,Acceptance rates averaged over 10 trials}
        \label{tab:clutter}
\end{table}

%% file: conc.tex
\section{Conclusion}
We propose pliable rejection sampling (\PRS), an adaptive rejection sampling method that learns its proposal distribution.
While previous work on adaptive rejection sampling aimed at decreasing the area
between the proposal and the target by iteratively updating the proposals according to sampling,
we learn it using a kernel estimator.  We show that \PRS outperforms traditional
rejection sampling and fares well with recent \Astar.
Our main contribution is a high-probability guarantee on the number of accepted samples using \PRS, and a  guarantee that only
a provably negligible number of samples are rejected with respect to the budget.

Since \PRS only estimates the proposal once, a possible algorithmic extension of \PRS is to iteratively update the kernel estimate as we gather more samples.
While this would result in the same theoretical acceptance guarantee as \PRS, the empirical performance is likely to be better.

We have also shown how to improve the scalability of the method to handle moderate dimensions --- in high dimensions, one would still suffer from numerical and memory cost. 
However, under the discussed assumptions, we get a number of steps which is polynomial in $d$, as opposed to exponential in $d$.
Extending the method to even higher dimension is an interesting research direction.






%% file: app.tex

\appendix


\section{Proofs of the main results}\label{app:proof}

\subsection{Proof of Theorem~\ref{thm:dimd2}}

Using Assumptions~\ref{ass:dens2} and~\ref{ass:ker} and the fact that $\|\cdot\|_2 \leq \|\cdot\|_s$, since $s\leq 2$, we have that for any $x \in [0,A]^d$,
\begin{align}
\Big| \EE{\widehat f(x)}  - f(x) \Big|   &\leq \frac{c''A^d}{A^dh^d} \int_{\mathbb{R}^d} \left|K\left(\frac{y - x}{h}\right)\right| \|y-x\|_s^{s} dy \nonumber \\
&= \frac{c''}{h^d} \sum_{i=1}^d \int_{\mathbb{R}^d} |y_i-x_i|^{s} \prod_{j=1}^d K_0\left(\frac{y_j - x_j}{h}\right) dy_j  \nonumber \\
&= \frac{c''}{h^d} \sum_{i=1}^d \int_{\mathbb R} K_0\left(\frac{y_i - x_i}{h}\right) |y_i-x_i|^{s} dy_i \prod_{j\neq i} \underbrace{\int_{\mathbb R} K_0\left(\frac{y_j - x_j}{h}\right) dy_j}_{=h} \nonumber \\
&= \frac{c''}{h} \sum_{i=1}^d\int_{\mathbb R} K_0\left(\frac{y_i - x_i}{h}\right) |y_i-x_i|^{s} dy_i \nonumber \\
&= c''d h^{s}\int_{\mathbb R} K_0(u) |u|^{s} du. \nonumber
\end{align}
Since $0 < s \leq 2$, we have that $ s \log|u| \leq \max(0, 2\log|u|)$, which means that $|u|^s \leq \max(1, |u|^2) \leq 1+|u|^2$. Therefore, we can write
\[\int_{\mathbb R} K_0(u) |u|^{s} du \leq \int_{\mathbb R} K_0(u) \left(1+|u|^2\right) du \leq 1 + C'.\]
Thus,
\begin{equation}
\Big| \EE{\widehat f(x)}  - f(x) \Big| \leq c'' (1 + C') d h^{s}.
\label{eq:Bias1}
\end{equation}
Let $0<\delta< \frac{1}{2}$. Note that since $K_0$ is non-negative and bounded by $C$, we have that $K \leq C^d$.
For $y \in \mathbb R^d$, let us write $Y_i = A^df(X_i)K\left(\frac{y - X_i}{h}\right)$, where $X_i \sim \mathcal U([0,A^d])$. This implies that $|Y_i| \leq c(CA)^d$. Moreover, since $f \leq  c$, we have
\begin{align*}
\mathbb V(Y_i) \leq \EE{Y_i^2} &= \frac{A^{2d}}{A^d}\int_{[0,A]^d} K^2\left(\frac{y - x}{h}\right) f^2(x) dx\\
&\leq c^2A^d\int_{\mathbb R^d} K^2\left(\frac{y - x}{h}\right) dx\\
&\leq c^2 A^d h^d \int_{\mathbb R^d} K^2(u) du\\
&\leq c^2C^d A^d h^d \int_{\mathbb R^d} K(u) du \\
&= c^2A^dC^d h^d.
\end{align*}

%
Therefore, by Bernstein's inequality, for any $x \in [0,A]^d$, we know that with probability larger than $1-\delta$ 
\begin{align}
\Big|  \EE{\widehat f(x)} - \widehat f(x)  \Big| 
&= \Bigg| \frac{A^d}{Nh^d} \sum_{k=0}^n \left(f(X_i)K\left(\frac{X_i - x}{h}\right) 
- \EE{f(X_i)K\left(\frac{X_i - x}{h}\right)} \right) \Bigg|\nonumber\\
&\leq  \frac{1}{Nh^d} \left(2c \sqrt{C^dA^d h^d N \log(1/\delta)} + 2cA^dC^d \log(1/\delta)\right)\nonumber\\
&\leq 2c\sqrt{C^dA^d\frac{\log(1/\delta)}{Nh^d}} + 2cA^d C^d \frac{\log(1/\delta)}{Nh^d},\label{eq:Variance1}
\end{align}
for $n$ large enough with respect to $\delta$.

By Equations~\ref{eq:Bias1} and \ref{eq:Variance1}, we have for any $x\in [0,A]^d$, with probability larger than $1-\delta$, that
\begin{align}
\left|\widehat f(x) -  f(x)\right| & \leq  2c\sqrt{C^dA^d\frac{\log(1/\delta)}{Nh^d}} 
+ 2c C^dA^d \frac{\log(1/\delta)}{Nh^d} +  c'' (1 + C') d h^{s}.\label{eq:Risk1}
\end{align}
Therefore, for $h \eqdef h_s(\delta) = \left(\frac{\log(NA/\delta)}{N}\right)^{\frac{1}{2s+d}}$, we get that
with probability larger than $1-\delta$,
\begin{align}
\left|\widehat f(x) -  f(x)\right|  \leq \left(2c\sqrt{C^dA^d}+c'' (1 + C') d \right) \left(\frac{\log(NA/\delta)}{N}\right)^{\frac{s}{2s+d}} 
 + 2 c C^d A^d \left(\frac{\log(NA/\delta)}{N}\right)^{\frac{2s}{2s+d}}.\label{eq:eqRisk1}
\end{align}

Let $\mathcal X$ be a $1/N^{v}$ covering set in $\|.\|_2$ norm and of minimal cardinality of the hypercube $[0,A]^d$. Its cardinality is at most $(AN^v\sqrt{d})^d\leq (ANd)^{d(v+1)}$, since the covering number of $[0,A]^d$ with small hypercubes of side $N^{-v}/\sqrt{d} $ is smaller than $\left(A / \frac{N^{-v}}{\sqrt{d}}\right)^{d} = (AN^v\sqrt{d})^d$,
and  hypercubes of side $N^{-v}/\sqrt{d}$
are contained in $\ell_2$ balls of diameter $N^{-v}$. By a union bound and Equation~\ref{eq:eqRisk1}, it holds that with probability larger than $1-\delta$, for any $x \in \mathcal X$, we have
\begin{align}
\left|\widehat f(x) -  f(x)\right| &\leq 4d(v+1)\left(2c\sqrt{A^dC^d}+c'' (1 + C') d\right) \nonumber
 \left(\frac{\log(NAd/\delta)}{N}\right)^{\frac{s}{2s+d}} \nonumber
\!\!\!+ 4cd(v+1) A^dC^d \left(\frac{\log(NAd/\delta)}{N}\right)^{\frac{2s}{2s+d}}. 
\end{align}
Let $\xi$ be the event of probability larger than $1-\delta$ where this is satisfied.
Let $y \in [0,A]^d$. Then, there exists $x \in \mathcal X$ such that $\|x - y\|_2\leq 1/N^{v}$. Since $K_0$ is $\varepsilon$-H\"older and since $f$ is bounded by $c$, we have that
\begin{align*}
\left|\widehat f(x) - \widehat f(y)\right| \leq cN(C'')^d\frac{c^d}{h_s(\delta)^d N^{d \varepsilon v}}
\leq (C'')^dc^{d+1}\frac{N^{1 + \frac{d}{2s+1}}}{ N^{d \varepsilon v}},
\end{align*}
and for $v > 3/\varepsilon \log(1 + 1/(C''c))$, we get
$$\left|\widehat f(x) - \widehat f(y)\right| \leq N^{-1}.$$
In the same way, by Assumption~\ref{ass:dens}, we get also that for $v \geq 2\log(1 + 1/(c''+c))/\min(1,s)$,
$$\left| f(x) -  f(y)\right| \leq c'' N^{-v} + c N^{-vs} \leq N^{-1}.$$
This implies that
$$\left|\widehat f(y) - f(y)\right| \leq \left|\widehat f(x) - f(x)\right| +2/N,$$
which means that on $\xi$, for any $x \in [0,A]^d$ we have
\begin{align*}
\left|\widehat f(x) -  f(x)\right| &\leq  4d(v+1) \left(2\sqrt{cA^dC^d}+cdC'\right)
 \left(\frac{\log(NAd/\delta)}{N}\right)^{\frac{s}{2s+d}}
\!\!\!\!+ 4cd(v+1) A^dC^d \left(\frac{\log(NAd/\delta)}{N}\right)^{\frac{2s}{2s+d}}+ 2/N,
\end{align*}
where $v = \log\left(1 + \frac{1}{c''+c}\right)\frac{2}{\min(1,s)}+ \frac{3}{\varepsilon} \log\left(1 + \frac{1}{C''c}\right)$. Therefore, we get that for any $x \in [0,A]^d$,
\begin{align*}
\Big| & \widehat f(x) -  f(x)\Big| \\
& \leq  4d(v+1)\left(2\sqrt{cA^dC^d}+c'' (1 + C') d \right) \left(\frac{\log(NAd/\delta)}{N}\right)^{\frac{s}{2s+d}}
+ 4cd(v+1) A^dC^d \left(\frac{\log(NAd/\delta)}{N}\right)^{\frac{2s}{2s+d}}+ 2/N \\
&\leq \left(2\sqrt{cA^dC^d}+c'' (1 + C') d  + c(AC)^d \left(\frac{\log(NAd/\delta)}{N}\right)^{\frac{s}{2s+d}}\right)
 8d(v+1)  \left(\frac{\log(NAd/\delta)}{N}\right)^{\frac{s}{2s+d}},\\
& \leq H_0 \left(\frac{\log(NAd/\delta)}{N}\right)^{\frac{s}{2s+d}},
\end{align*}
where $v = \log\left(1 + \frac{1}{c''+c}\right)\frac{2}{\min(1,s)}+ \frac{3}{\varepsilon} \log\left(1 + \frac{1}{C''c}\right)$ and $H_0$ a constant that depends on $d, v, c, c'', C, C'$, and $A$.

\section{Extension to densities with unbounded support}\label{app:unb}


In this part of the appendix, we extend our method to densities that do not have a compact support. For the sake of clarity, we assume that $f$ is normalized here. In fact, we have already shown how to deal with the unnormalized case in the Section~\ref{algo_results}, where the normalization constant is estimated by a Monte-Carlo sum over the same samples used to estimate $f$.

\begin{assumption}[Assumption on the density]\label{ass:dens}
The density $f$, defined on $\mathbb R^d$, is sub-Gaussian, i.e.,~there exist constants $c,c' > 0$ such that the density $f$ satisfies for any $x \in \mathbb R^d$
$$f(x) \leq c \exp\left(-c' \|x\|_2^2\right).$$

Moreover, $f$ can be uniformly expanded by a Taylor expansion in any point up to degree $s$ for some $0<s\leq 2$, i.e.,~there exists $c'' >0$ such that, for any $x\in \mathbb R^d$, and for any $u\in \mathbb R^d$, we have
$$\left|f(x + u) - f(x) - \langle \bigtriangledown f(x), u\rangle \mathbf 1\{s >1\}\right| \leq c'' \|u\|_2^{s}.$$
\end{assumption}

The above assumption means that the tails of $f$ are sub-Gaussian, and also that $f$ is in a H\"older ball of smoothness $s$. Note that a bounded function $f$ with a compact support in $\mathbb R^d$ is sub-Gaussian. The fact that $f$ is in a H\"older ball of smoothnes~$s$ is also not very restrictive, in particular for a small $s$.

\subsection{Uniform bounds for kernel density estimation}\label{ss:estf}
Let $X_1, \ldots, X_N$ be $N$ points generated by $f$. Let us define for $h \eqdef h_s(\delta) = \left(\frac{\log(N/\delta)}{N}\right)^{\frac{1}{2s+1}}$,
$$\widehat f(x) = \frac{1}{Nh^d} \sum_{k=1}^N K\left(\frac{X_i - x}{h}\right),$$
and $\tilde f$ be such that
\begin{eqnarray}
\tilde f(x) = \widehat f(x) \mathbf 1\left\{\|x\|_2 \leq \log(N)\right\}.\label{eqn:tildef}
\end{eqnarray}

\begin{theorem}\label{thm:dimd}
Assume that Assumptions~\ref{ass:ker} and~\ref{ass:dens} hold with $0<s \leq 2$, $C,C',C'',c,c',c''>0$, and $\varepsilon >0$. The estimate $\tilde f$ is such that with probability larger than $1-\delta$, for any point $x \in \mathbb R^d$,
\begin{align*}
\left|\tilde f(x) -  f(x)\right| &\leq 8d(v+1)(2\sqrt{cC^d}+c'' (1 + C') d+C^d) \left(\frac{\log(N/\delta)}{N}\right)^{\frac{s}{2s+d}}
 + c\exp\left(-c' \|x\|_2^2\right) \mathbf 1_{\|x\|_2 \geq \log(N)}\\
&\leq  H_1  \left(\frac{\log(N/\delta)}{N}\right)^{\frac{s}{2s+d}}
 + c\exp\left(-c' \|x\|_2^2\right) \mathbf 1_{\|x\|_2 \geq \log(N)}
\end{align*}
where $v = \log\left(1 + \frac{1}{c''+c}\right)\frac{2}{\min(1,s)}+ \frac{3}{\varepsilon} \log\left(1 + \frac{1}{C''c}\right)$, and $H_1$ is a constant that depends on $d, v, c, c'', C, C'$.
\end{theorem}
Theorem~\ref{thm:dimd} provides a uniform bound on the error of $\tilde f$ on a large centered ball of radius $\log n$ denoted by $B_d\left(\log n)\right)$, and bounds the far fluctuations by an upper bound on $f$ itself. Note that the previous bound implies in particular that
$$\left|\tilde f(x) -  f(x)\right| \leq \cO\left( \left(\frac{\log(N/\delta)}{N}\right)^{\frac{s}{2s+d}}\right).$$

\subsection{Proof of Theorem~\ref{thm:dimd}}


By Assumptions~\ref{ass:ker} and~\ref{ass:dens}, similarly to the starting point of proof of Theorem \ref{thm:dimd2} (see Appendix \ref{app:proof}), we have that for any $x \in \mathbb R^d$
\begin{align}
\Big| \EE{\widehat f(x)}  - f(x) \Big| \leq c'' (1 + C') d h^{s}.
\label{eq:Bias2}
\end{align}

Let $0<\delta< \frac{1}{2}$. Note that since $K_0$ is non-negative and bounded by $C$, we have that $K \leq C^d$.
For $y \in \mathbb R^d$, let us write $Y_i = K\left(\frac{y - X_i}{h}\right)$, where $X_i \sim f$. This implies that $|Y_i| \leq C^d$. Moreover, since $f(x) \leq c \exp\left(-c'\|x\|_2^2\right) \leq c$, we have
\begin{align*}
\mathbb V(Y_i) &\leq \int_{\mathbb R^d} K^2\left(\frac{y - x}{h}\right) f(x) dx\\
&\leq c\int_{\mathbb R^d} K^2\left(\frac{y - x}{h}\right) dx\\
&\leq c h^d \int_{\mathbb R^d} K^2(u) du\\
&\leq cC^d h^d \int_{\mathbb R^d} K(u) du \\
&= cC^d h^d.
\end{align*}

%
Therefore, by Bernstein's inequality, for any $x \in \mathbb R^d$, we know that with probability larger than $1-\delta$ 
\begin{align}
\Big|  \EE{\widehat f(x)} - \widehat f(x)  \Big|
&= \left| \frac{1}{Nh^d} \sum_{k=0}^n \left(K\left(\frac{X_i - x}{h}\right) - \EE{K\left(\frac{X_i - x}{h}\right)} \right) \right|\nonumber\\
&\leq  \frac{1}{Nh^d} \Big(2 \sqrt{cC^d h^d N \log(1/\delta)} + 2 C^d \log(1/\delta)\Big)\nonumber\\
&\leq 2\sqrt{cC^d\frac{\log(1/\delta)}{Nh^d}} + 2 C^d \frac{\log(1/\delta)}{Nh^d},\label{eq:Variance2}
\end{align}
for $n$ large enough with respect to $\delta$.

By Equations~\ref{eq:Bias2} and \ref{eq:Variance2}, we thus know that for any $x \in \mathbb R^d$, with probability larger than $1-\delta$,
\begin{align}
\left|\widehat f(x) -  f(x)\right| & \leq  2\sqrt{cC^d\frac{\log(1/\delta)}{Nh^d}}
+ 2 C^d \frac{\log(1/\delta)}{Nh^d} + c'' (1 + C') d h^{s}.\label{eq:Risk2}
\end{align}
Therefore, for $h = h_s(\delta) = \left(\frac{\log(N/\delta)}{N}\right)^{\frac{1}{2s+d}}$, we get that
with probability larger than $1-\delta$,
\begin{align}
\left|\widehat f(x) -  f(x)\right| & \leq \left(2\sqrt{cC^d}+c'' (1 + C') d\right) \left(\frac{\log(N/\delta)}{N}\right)^{\frac{s}{2s+d}}
 + 2 C^d \left(\frac{\log(N/\delta)}{N}\right)^{\frac{2s}{2s+d}}.\label{aaaaa}
\end{align}

Now, since the $\Psi_2$ norm of $f$ is bounded by $c'$, we know that for any $x \in \mathbb R^d$,
\begin{equation*}
f(x) \leq c\exp\left(-c' \|x\|_2^2\right).
\end{equation*}

This implies in particular that for any $x$ such that $\|x\|_2 \geq \log(N)$, we have
\begin{equation}\label{outlogn2}
f(x) \leq c\exp\left(-c' \|x\|_2^2\right) \mathbf 1_{\|x\|_2 \geq \log N}.
\end{equation}

Let $\mathcal X$ be a $1/N^{v}$ covering set in $\|.\|_2$ norm and of minimal cardinality of the ball of $\mathbb R^d$ of center $0$ and radius $\log N$ that we will denote by $B_d(\log N)$. Its cardinality is at most $(2N^v\log N)^d\leq N^{d(v+1)}$ (by a similar reasoning as in the proof of Thm~\ref{thm:dimd2}). By a union bound and Equation~\ref{aaaaa}, it holds that with probability larger than $1-\delta$, for any $x \in \mathcal X$, we have
\begin{align}
\left|\widehat f(x) -  f(x)\right| &\leq 4d(v+1)\left(2\sqrt{cC^d}+c'' (1 + C') d\right) \left(\frac{\log(N/\delta)}{N}\right)^{\frac{s}{2s+d}}
+ 4d(v+1) C^d \left(\frac{\log(N/\delta)}{N}\right)^{\frac{2s}{2s+d}}\cdot\label{eq:eqRiskcov21}
\end{align}
Let $\xi$ be the event of probability larger than $1-\delta$ where this is satisfied.
Let $y \in B_d(\log N)$. Then, there exists $x \in \mathcal X$ such that $\|x - y\|_2\leq 1/N^{v}$. Since $K_0$ is $\varepsilon$-H\"older, we have  that
$$\left|\widehat f(x) - \widehat f(y)\right| \leq N(C'')^d\frac{c^d}{h_s(\delta)^d N^{d \varepsilon v}} \leq (C'')^dc^d\frac{N^{1 + \frac{d}{2s+1}}}{ N^{d \varepsilon v}},$$
and for $v > 3/\varepsilon \log(1 + 1/(C''c))$, we get
$$\left|\widehat f(x) - \widehat f(y)\right| \leq N^{-1}.$$
In the same way, by Assumption~\ref{ass:dens}, we also get that for $v \geq 2\log(1 + 1/(c''+c))/(\min(1,s))$,
$$\left| f(x) -  f(y)\right| \leq c'' N^{-v} + c N^{-vs} \leq N^{-1}.$$
This implies that
$$\left|\widehat f(y) - f(y)\right| \leq \left|\widehat f(x) - f(x)\right| +2/N,$$
which means that on $\xi$, for any $x \in B_d(\log N)$ we have
\begin{align*}
\left|\widehat f(x) -  f(x)\right| &\leq  4d(v+1) \left(2\sqrt{cC^d}+c'' (1 + C') d\right) \left(\frac{\log(N/\delta)}{N}\right)^{\frac{s}{2s+d}}
+ 4d(v+1) C^d \left(\frac{\log(N/\delta)}{N}\right)^{\frac{2s}{2s+d}}+ 2/N,
\end{align*}
where $v = \log\left(1 + \frac{1}{c''+c}\right)\frac{2}{\min(1,s)}+ \frac{3}{\varepsilon} \log\left(1 + \frac{1}{C''c}\right)$. Combining this with the definition of $\tilde f$, and Equation~\ref{outlogn2}, we get that for any $x \in \mathbb R^d$,
\begin{align*}
\Big| & \tilde f(x) -  f(x)\Big| \\
& \leq  4d(v+1)\left(2\sqrt{cC^d}+c'' (1 + C') d\right) \left(\frac{\log(N/\delta)}{N}\right)^{\frac{s}{2s+d}}
\!\!\!\! + 4d(v+1) C^d \left(\frac{\log(N/\delta)}{N}\right)^{\frac{2s}{2s+d}} \!\!\!\! + 2/N
+  f(x)\mathbf 1_{\|x\|_2 \geq \log N}\\
&\leq 8d(v+1)\left(2\sqrt{cC^d}+c'' (1 + C') d\right) \left(\frac{\log(N/\delta)}{N}\right)^{\frac{s}{2s+d}}
\!\!\!\! + 4d(v+1) C^d \left(\frac{\log(N/\delta)}{N}\right)^{\frac{2s}{2s+d}}+ f(x)\mathbf 1_{\|x\|_2 \geq \log N}\\
&\leq 8d(v+1)\left(2\sqrt{cC^d}+c'' (1 + C') d\right) \left(\frac{\log(N/\delta)}{N}\right)^{\frac{s}{2s+d}}
\!\!\!\!\!\! +4d(v+1) C^d \left(\frac{\log(N/\delta)}{N}\right)^{\frac{2s}{2s+d}}
\!\!\!\!\!\! + c\exp\left(-c' \|x\|_2^2\right) \mathbf 1_{\|x\|_2 \geq \log N},
\end{align*}
where $v = \log\left(1 + \frac{1}{c''+c}\right)\frac{2}{\min(1,s)}+ \frac{3}{\varepsilon} \log\left(1 + \frac{1}{C''c}\right)$

\subsection{Extended pliable rejection sampling}

Our modified algorithm, \textit{\textbf{e}xtended \textbf{p}liable \textbf{r}ejection \textbf{s}ampling (\EPRS)}, aims at sampling as many i.i.d.~points distributed according to $f$ as possible with a fixed budget of evaluations of $f$. It consists of three main steps: (i) a first rejection sampling step where it generates $\widehat N$ initial samples from $f$ by rejection sampling using an \textit{initial proposal}. Then, (ii)  \EPRS uses these samples to estimate $f$ by a kernel density estimation method.  Finally, (iii) \EPRS uses the newly obtained estimate, plus a uniform bound on it, as a new \textit{extended pliable proposal} for rejection sampling. Since this pliable proposal is closer to the target density than the initial proposal, the rejection sampling will reject significantly fewer points by using it. Our \EPRS method is described as Algorithm~\ref{alg:1}.

As mentioned, our method makes use of an \textit{initial proposal} density $g$ that must satisfy the following properties with respect to the target density.
\begin{assumption}[Assumption on the initial proposal]\label{ass:ini}
Let $M>0$. We have
$$f \leq Mg.$$%

Furthermore, the density $g$ is sub-Gaussian, i.e.,~there exist constants $a,a' > 0$ such that the density $g$ defined on $\mathbb R^d$ satisfies
$$g(x) \leq a \exp\left(-a' \|x\|_2^2\right).$$
\end{assumption}

We set the following constants:
\[T_s \eqdef n^{\frac{2s+d}{3s+d}}, \quad \mathrm{and} \quad \bar N \eqdef T_s/M - 2\sqrt{T_s \log(1/\delta)}.\]
$\bar N$ is actually a high probability lower bound of $\widehat N$ given by our algorithm (it is the number of samples obtained by the initial rejection sampling step). $T_s$ is the number of calls needed for the first estimation step that will optimize the number of accepted samples in the second step.

We also define
$$r_{\bar N} \eqdef \mathcal V_n H_E \left(\frac{\log(\bar N/\delta)}{\bar N}\right)^{\frac{s}{2s+d}}, \quad \mathrm{and}$$ $$ \bar g_{\bar N} \eqdef \int_{B_d(\log(\bar N))^C}g(x) dx,$$
where $\mathcal V_n = \mathcal V\left(B_d\left(\log\left(n\right)\right)\right)$ is the volume of $B_d\left(\log n \right)$ the centered ball in $\mathbb R^d$ of radius $\log n$ and where $H$ is a parameter of the algorithm.


Our method samples most of the samples by rejection sampling according to a \textit{pliable proposal} that is defined as
$$\hat g^\star \eqdef \frac{1}{1 + r_{\bar N} + M\bar g_{\bar N}} \left(\tilde f + r_{\bar N} \mathcal U_{B_d(\log(n))} + Mg \mathbf 1_{\|x\|_2 \geq \log(\bar N)}\right),$$
where $\mathcal U_{B_d(\log(n))}$ is the uniform distribution on $B_d(\log(n))$, and $\tilde f$ the estimate of $f$ defined in \eqref{eqn:tildef}.

\begin{algorithm}[ht]
\begin{algorithmic}
\caption{Extended pliable rejection sampling (\EPRS)}
 \label{alg:1}
\STATE {\bf Parameters:} $s$, $n$, $\delta$, $H$ $g$, and $M$.

\STATE \textbf{Initial sampling}

\STATE \quad Draw $T_s$ samples at random according to $g$, and evaluate $f$ on them

\STATE \textbf{Estimation of $f$}

\STATE \quad Perform rejection sampling on the samples using $M$ as the constant
\STATE \quad  Obtain this way $\widehat N$ samples from $f$
\STATE \quad Estimate $f$ by $\tilde f$ on these $\widehat N$ samples (Section~\ref{ss:estf})

\STATE \textbf{Generating the sample}

\STATE \quad Sample $n- T_s$ samples from the \textit{pliable proposal} $\hat g^\star$

\STATE \quad Perform rejection sampling on these samples using $1 + r_{\bar N}/\mathcal V_n +M \bar g_{\bar N}$ as a constant
\STATE \quad  Obtain this way $\widehat n$ samples from $f$

\STATE \textbf{Output:} Return the $\widehat n$ samples.
\end{algorithmic}
\end{algorithm}

\begin{theorem}\label{thm:prs}
Assume that Assumptions~\ref{ass:ker}, \ref{ass:dens}, and~\ref{ass:ini} hold with $0<s \leq 2$, $g$, $M>0$, and that $H_E >0$ is an upper bound on the constant $H_1$ defined in Theorem~\ref{thm:dimd} (applied to $f$ and $\tilde f$). The number $\widehat n$ of samples generated in an i.i.d.~way according to $f$ in this way is such that for $n$ large enough, with probability larger than $1-\delta$,
$$\widehat n \geq n\left[1 -  \cO\left(\log\left(n/\delta\right)^{d+1}n^{-\frac{s}{3s+d}}\right)\right].$$
\end{theorem}

\subsection{Proof of Theorem~\ref{thm:prs}}

By definition of $\bar g_{\bar N}$ and since the $g$ is sub-Gaussian by Assumption~\ref{ass:dens}, we have that

\begin{align*}
\bar g_{\bar N} &= \int_{B_d(\log \bar N)^C}g(x) dx 
 \leq \int_{B_d(\log \bar N)^C}a \exp(-a' \|x\|_2^2) dx 
 \leq  a(\bar N)^{-da' (\log\bar N)/4} \leq \bar N^{-1},\label{eq:g}
\end{align*}
for $n$ (and thus $\bar N$) large enough.


By definition of rejection sampling, the probability of accepting a sample is exactly $1/M$, where $M$ is the upper bound used in the rejection sampling (provided that $f \leq Mg$).
Therefore, and since the first rejection sampling step uses $T_s$ samples, $\widehat N$ is a sum of $T_s$ independent Bernoulli random variables of parameter $1/M$. Thus, we have with probability larger than $1-\delta$ that
\begin{equation}
  \widehat N \geq T_s/M - 2\sqrt{T_s \log(1/\delta)} \eqdef \bar N.
\label{eq:sumBernoulli}
\end{equation}
Let us write $\xi'$ for the event where this happens.
On $\xi'$, we have by Theorem~\ref{thm:dimd} (end of the proof) that with probability larger than $1-\delta$, for any $x \in \mathbb R^d$
$$\left|\tilde f(x) - f(x)\right| \leq \frac{r_{\widehat N}}{ \mathcal V_n} +  f \mathbf 1_{\|x\|_2 \geq \log \widehat N} \leq \frac{r_{\bar N}}{ \mathcal V_n} +   M g \mathbf 1_{\|x\|_2 \geq \log \bar N}.$$

Let $\xi$ be the intersection of $\xi'$ and the event where Equation \ref{eq:sumBernoulli} holds. It has probability larger than $1-2\delta$. 
On $\xi$, we thus have that
$$ \tilde f + r_{\bar N} \mathcal U_{B_d(\log n)} +  Mg \mathbf 1_{\|x\|_2 \geq \log \bar N} \geq f.$$

Therefore, the rejection sampling is going to provide samples that are i.i.d.~according to $f$, and $\widehat n$ will be a sum of Bernoulli random variables of parameter $\frac{1}{1 + r_{\bar N} + M\bar g_{\bar N}}$. We thus have that on $\xi$, with probability larger than $1-\delta$,
$$\widehat n \geq (n-T_s)\frac{1}{1 + r_{\bar N} + M\bar g_{\bar N}} - 2 \sqrt{n\log(1/\delta}).$$

This implies, together with the definition of $r_{\bar N}$ and the upper bound on $\bar g_{\bar N}$, that $\widehat n$ is with probability larger than $1-3\delta$ lower bounded as
\begin{align}
\widehat n & \geq  (n - T_s) \left(1 -  \pi^d\log\left(\bar N\right)^dH \left(\frac{\log(\bar N/\delta)}{\bar N}\right)^{\frac{s}{2s+d}} -  M\bar N^{-1} \right) - 2\sqrt{n\log(1/\delta}) \nonumber\\
& \geq n\left[1 - \frac{T_s}{n} - \pi^d\log\left(n\right)^d H \left(\frac{\log(n/\delta)}{T_s/M - 2\sqrt{T_s \log(1/\delta)}}\right)^{\frac{s}{2s+d}} - \frac{M}{T_s/M - 2\sqrt{T_s \log(1/\delta)}} -  4\sqrt{\frac{\log(1/\delta)}{n}}\right] \nonumber.
\end{align}

Since
$$T_s = n^{\frac{2s+d}{3s+d}},$$
we have that with probability larger than $1-3\delta$, for $n$ large enough, there exists a constant $K$ such that
\begin{align}
\widehat n \geq n\left[1 - K\log(n/\delta)^{d+1}n^{-\frac{s}{3s+d}}. \right]
\end{align}